\documentclass[runningheads]{llncs}
\usepackage{graphicx}
\usepackage{amsmath,amssymb} 
\usepackage{color}
\usepackage[width=122mm,left=12mm,paperwidth=146mm,height=193mm,top=12mm,paperheight=217mm]{geometry}

\usepackage{hyperref}

\makeatletter
\newcommand{\printfnsymbol}[1]{%
  \textsuperscript{\@fnsymbol{#1}}%
}
\makeatother

\usepackage{amsfonts}
\usepackage{bbm}
\usepackage{algorithm}
\usepackage{algpseudocode}
\newtheorem{prop}{Proposition}
\usepackage{multirow}
\usepackage{caption}
\usepackage{subcaption}
\usepackage{booktabs}

\begin{document}
\pagestyle{headings}
\mainmatter
\def\ECCV16SubNumber{***}  

\title{Certifying Global Robustness\\ for Deep Neural Networks} 

\titlerunning{Certifying Global Robustness for Deep Neural Networks}

\authorrunning{Y. Li, G. Zhao, S. Kong, Y. He and H. Zhou}

\author{You Li\thanks{Contributed equally.}\inst{1} \and
Guannan Zhao\printfnsymbol{1} \inst{1} \and
Shuyu Kong\inst{2} \and
Yunqi He\inst{1} \and
Hai Zhou\inst{1}}

%
\institute{Northwestern University, USA\\
\email{\{you.li, gnzhao, yunqi.he\}@u.northwestern.edu}\\
\email{haizhou@northwestern.edu}
\and
Meta\\
\email{kshuyu@meta.com}
}

\maketitle

\begin{abstract}
A globally robust deep neural network resists perturbations on all meaningful inputs.
Current robustness certification methods emphasize local robustness, struggling to scale and generalize.
This paper presents a systematic and efficient method to evaluate and verify global robustness for deep neural networks, leveraging the PAC verification framework for solid guarantees on verification results. 
We utilize probabilistic programs to characterize meaningful input regions, setting a realistic standard for global robustness.
Additionally, we introduce the cumulative robustness curve as a criterion in evaluating global robustness. We design a statistical method that combines multi-level splitting and regression analysis for the estimation,
significantly reducing the execution time. 
Experimental results demonstrate the efficiency and effectiveness of our verification method and its capability to find rare and diversified counterexamples for adversarial training.

\end{abstract}

\section{Introduction}

Deep Neural Networks (DNNs) have achieved remarkable success in numerous machine learning tasks, yet they are vulnerable to adversarial attacks. Given a correctly classified input example $x$, an adversary can craft a small perturbation $\Delta$, such that for $x+\Delta$, the targeted DNN will produce a prediction that violates a property of interest. Meanwhile, $x+\Delta$ is almost indistinguishable from $x$ to the human eye~\cite{goodfellow2014explaining}. Adversarial attacks have posed significant threats to decision-critical systems like autonomous driving~\cite{sun2020towards}, critical infrastructures~\cite{sayghe2020evasion}, face recognition~\cite{dong2019efficient}, etc.

Various heuristic methods were proposed to defend against adversarial attacks. Nevertheless, subsequent attacks can always defeat them~\cite{athalye2018obfuscated,carlini2017adversarial}. In response, researchers propose \textit{certifiable robustness} to provide rigorous mathematical guarantees against those attacks with $\Delta$ bounded by some $l_p$-norm. 
Intuitively, it ensures that a classifier's predictions are consistent within the neighbourhood of a specific input example.
These methods either use layer-by-layer reachability analysis on a network or employ optimization and constraint solving techniques to check the existence of adversarial examples. Unfortunately, existing robustness certification methods still face challenges in scaling to large DNNs when checking a single input example and its neighborhood~\cite{liu2019algorithms}, let alone generalizing to  entire input regions.

Goldwasser \textit{et al.}~\cite{goldwasser2021interactive} recently proposed and laid the theoretical foundation for \textit{PAC Verification}.
We adapt PAC verification to the task of certifying the global robustness of DNNs. A significant challenge here is how to characterize meaningful input regions. The lack of a global characterization has restricted the effectiveness of DNN verification in the following ways: \textit{i)} local robustness verification cannot be directly generalized to the whole decision region of a class; \textit{ii)} a verification task is limited within an $l_p$-norm distance from the input $x$, while an adversary is not restricted; \textit{iii)} an $l_p$-norm ball could cross decision boundaries or contain meaningless input regions, making a verification task unrealistic. Through an analysis of robustness metrics, we discover that a probabilistic program can address all the above issues.

Nevertheless, typical neural networks are not strictly robust in the entire input regions.
Instead, we consider the optimal robustness value for a specific DNN given some robustness metric. We also let the verifier specify a tolerable threshold for local robustness.
We summarize the verification result in \textit{cumulative robustness curves}, which depict the robustness value as a function of perturbation radius, robustness metric and local robustness threshold. The curves serve as a comprehensive measure of global robustness for a DNN. When the function is evaluated at a certain point, its output directly indicates whether the DNN satisfies prescribed requirements.

It is nontrivial to determine whether the local robustness of an input sample meets a threshold, because the probability of violating a robustness metric could be extremely small. We devise the margin function to assess whether a robustness metric is violated. Facilitated by the margin function, we employ adaptive multi-level splitting to measure local robustness. Additionally, we design a parameter estimation technique to estimate the same quantity. In our experiments, we find that a combination of the two techniques achieves the best balance between efficiency and accuracy.

The main contributions of this work are as follows: 1) we formularize and analyze the global robustness risk for a DNN; 2) we suggest using probabilistic programs to characterize the global input distribution, such that global robustness verification is feasible; 3) we propose the cumulative robustness function as a comprehensive measure of global robustness; 4) we devise adaptive multi-level splitting calibrated estimation (ACE) to efficiently and accurately estimate local robustness for samples in a specific global distribution.
\section{Global Robustness of Deep Neural Networks}

\subsection{Preliminaries}
Let $f_{\theta}(\cdot)$ represent a DNN classifier. Given the parameters $\theta$, it can be characterized as a function $f_{\theta}: \mathbb{R}^{m} \rightarrow \mathbb{R}^{n}$ where $m$ is the dimension of input samples and $n$ is the number of output classes.
A classifier typically maps an input sample $x$ to a vector of \textit{scores} $[y_1(x), y_2(x), \cdots, y_n(x)]$,
and the corresponding
output label $f_{\theta}(x)$ is the greatest element among the output vector: $f_{\theta}(x) = \texttt{argmax}_{ i \in \{ 1, \cdots,  n\} }\, y_i(x)$.
On the other hand, let $c(x)$ be the ground truth function, which gives the underlying true labels for the input samples.

The decision region of class $i$ is the set of all input samples whose $i$-th score is the greatest:
$\Omega_{f_{\theta},i}=\{x\, |\, y_i(x)-\texttt{max}_{j\neq i}\, y_j(x) > 0\}$,
while the decision boundary $\textit{DB}_{f_{\theta}}$ partitions different decision regions.
Accordingly, $\textit{DB}_{c}$ denotes the ground truth boundary. Let $\mathbb{B}(x,r)$ represent the $r$-neighbourhood of the input sample $x$ or the input distribution within the $r$-neighbourhood, and let $\mathbb{B}(F,r)$ be the union of the $r$-neighbourhoods for all input samples satisfying some formula $F$.
\\

\subsection{Analysis of Global Robustness}The most commonly used model for the adversarial attack is the additive model: a real input $x'$ consists of two terms, the \textit{nominal input} $x$, and an additional perturbation $\Delta \in \mathbb{R}^{m}$ bounded by radius $r$ with regard to some $l_p$ norm: $\left \Vert \Delta \right \Vert_p < r$. 
The local robustness property of DNN can thus be defined as follows:

\begin{definition} [Local Robustness]
\label{local_rob}
Given a nominal input $x \in \mathbb{R}^{m}$, a DNN $f_{\theta}: \mathbb{R}^{m} \rightarrow \mathbb{R}^{n}$ is locally robust at $x$ with regard to radius $r$, 
if $\forall$$x' \in \mathbb{B}(x,r)$:
$m(f_{\theta}, x, x') = \mathtt{true}$.
\end{definition}

We term $m(\cdot)$ as \textit{robustness metric}. Three metrics are commonly used to measure local robustness~\cite{diochnos2018adversarial,zhang2019theoretically}:
\begin{equation}
  m(f_{\theta}, x, x') \triangleq
    \begin{cases}
      \ m_0: f_{\theta}(x') = c(x'), & \\
      \ m_1: f_{\theta}(x') = c(x), & \\
      \ m_2: f_{\theta}(x') = f_{\theta}(x). & 
    \end{cases}       
\end{equation}

$m_0$ mainly concerns \textit{accuracy} and is only measurable when ground truth labels are available everywhere in the neighbourhood. $m_2$ concerns the \textit{consistency} of prediction, whereas $m_1$ combines the other two metrics by requiring the nominal input and all perturbed inputs within the neighborhood to produce the same correct label.

Definition~\ref{local_rob} ensures the absence of particular types of adversarial examples around a certain input sample. One may generalize this definition to the whole decision region of a class by requiring it to hold everywhere in the region. Nevertheless, for many applications, such a fully robust DNN is not realistically obtainable. 
Additionally, statistical defense techniques can provide guaranteed robustness when adversarial examples are rare within a distribution~\cite{cohen2019certified,dhillon2018stochastic,wong2018provable}.
Therefore, instead of requiring the DNN to be robust everywhere, we measure the probability that a certain robustness metric is violated. The \textit{robustness risk} with regard to an input distribution $\mathcal{D}$ is defined as 
\begin{equation}
\mathbf{R}_{rob}(f_{\theta},m) \triangleq \mathbf{E}_{x \sim \mathcal{D}}\mathbbm{1}[\exists x' \in \mathbb{B}(x,r): m(f_{\theta}, x, x') = \mathtt{false}].
\label{risk}
\end{equation}

The robustness risk can be further decomposed.
Define classification risk $\mathbf{R}_{c}(f_{\theta}) \triangleq \mathbf{E}_{x \sim \mathcal{D}}\mathbbm{1}[f_{\theta}(x)\neq c(x)]$, which is the standard measure of the error rate of a classifier. Let boundary risk~\cite{zhang2019theoretically} represent the probability that a sample is classified correctly but resides within the neighbourhood of the decision boundary: $\mathbf{R}_{b}(f_{\theta}) \triangleq \mathbf{E}_{x \sim \mathcal{D}}\mathbbm{1}[f_{\theta}(x)=c(x) \wedge x \in \mathbb{B}(\textit{DB}_{f_{\theta}},r)]$. To deal with global robustness, we introduce ground truth boundary risk defined as $\mathbf{R}_{gb}(f_{\theta}) \triangleq \mathbf{E}_{x \sim \mathcal{D}}\mathbbm{1}[f_{\theta}(x)=c(x) \wedge x \in (\mathbb{B}(\textit{DB}_{c},r) \setminus \mathbb{B}(\textit{DB}_{f_{\theta}},r))]$. From these definitions, we can derive the following relations:
\begin{equation}
    \begin{cases}
      \ \mathbf{R}_{rob}(f_{\theta},m_0) \approx \mathbf{R}_{c}(f_{\theta}) + \mathbf{R}_{b}(f_{\theta}) + \mathbf{R}_{gb}(f_{\theta}), & \\
      \ \mathbf{R}_{rob}(f_{\theta},m_1) = \mathbf{R}_{c}(f_{\theta}) + \mathbf{R}_{b}(f_{\theta}), & \\
      \ \mathbf{R}_{rob}(f_{\theta},m_2) \leq \mathbf{R}_{c}(f_{\theta}) + \mathbf{R}_{b}(f_{\theta}). & 
    \end{cases}  
\label{approx}
\end{equation}

$\mathbf{R}_{rob}(f_{\theta},m_0)$ equals the right hand side if $\mathbb{B}(\textit{DB}_{f_{\theta}},r)$ and $\mathbb{B}(\textit{DB}_{c},r)$ are non-overlapping.
Besides, they are almost equal when the distributions of the two neighbourhoods are independent, as the probability of overlapping is negligible.
$\mathbf{R}_{rob}(f_{\theta},m_2)$ gets infinitely close to the right-hand side as the value of $r$ increases because then the $r$-neighbourhood of the decision boundary covers more mis-classified samples. Both $\mathbf{R}_{b}(f_{\theta})$ and $\mathbf{R}_{gb}(f_{\theta})$ converge to $0$ as $r$ approaches $0$, in which case the robustness risk is equal to the classification risk.

\subsection{Characterizing the Global Input Distribution}

In addition to robustness metrics, the input distribution $\mathcal{D}$ is another crucial factor in evaluating the global robustness risk. A naive solution is to set $\mathcal{D}$ as the uniform distribution in the whole input space. Such a distribution does not correspond to meaningful regions, and the counterexamples returned from verification may not be meaningful instances. Alternatively, $\mathcal{D}$ can be set to the whole decision regions of classes of $f_{\theta}$. However, these regions could intersect with the ground truth boundaries, $\textit{DB}_{c}$. Moreover, $\textit{DB}_{f_{\theta}}$ itself could have a significant impact on the robustness risk and should not be neglected.

An ideal mechanism to characterize the global input distribution $\mathcal{D}$ should guarantee that \textit{i)} $\mathcal{D}$ contains and concentrates on the meaningful regions as humans do; \textit{ii)} $\mathcal{D}$ has a minimal intersection with the ground truth boundaries; \textit{iii)} we can draw a large number of samples efficiently from $\mathcal{D}$. We believe probabilistic programs can address all these concerns.

A probabilistic program~\cite{lake2015human,feinman2020learning} is a hierarchical model composed of random variables. Each random variable is a distribution or a conditional distribution depending upon other random variables. The program itself is a generative model which takes a class label $\psi$ as a parameter and produces a sample $I$. It executes by drawing random values from the distributions following the hierarchy of the program as well as the conditional relations between random variables:
$$I := G(\phi, R, \psi)$$
where $\phi$ denotes the parameters of distributions of all random variables, while $R$ captures the overall hierarchy and the conditional relations. A new sample of the corresponding class is produced in each forwarding pass of the program. 

A probabilistic program learns by fitting the conditional distributions to the training examples. Formally speaking, given a set of examples $M$, the learning task is to optimize $\phi$ to maximize the joint distribution of all sample-label pairs: $\phi^\star = \mathtt{argmax}_{\phi}\prod_{m \in M}P(\phi \, | \, I_{m},\psi_{m})$.
With this optimized $\phi^\star$, new samples drawn from the probabilistic program follows a meaningful input distribution.

Cognitive scientists have shown that probabilistic programs capture causal and compositional relations in a similar way as humans do~\cite{goodman2014concepts}. As each program represents a concept, it is very efficient to generate new random examples within a designated class. More importantly, probabilistic programs hold strong inductive bias. Thus, domain experts can embed their prior domain knowledge into the overall hierarchy, relations of random variables, and the selection of distributions, such that their developed programs can represent concepts and therefore accurately characterize meaningful input regions instead of the whole decision regions. We will further elaborate on the necessity of probabilistic programs in verifying global robustness in Section~\ref{statistical}.
\section{Statistical Global Robustness Verification}
\label{statistical}

In this section, we present PAC robustness verification, a statistical proof system for global robustness. A DNN is certified if its robustness risk is below a user-specified threshold or it is close to the optimal value. Then we improve verification efficiency by introducing an algorithm called ACE. In addition, our proposed method can efficiently capture diversified counterexamples even if the chance of violation is extremely rare.

\subsection{PAC Robustness Verification}
\label{pac}

PAC verification~\cite{goldwasser2021interactive} was recently proposed by Goldwasser \textit{etc.}, to verify machine learning models.
It relaxes the correctness condition by allowing a minimal error. A hypothesis is accepted if the error is smaller than a sufficiently small $\epsilon$ with high confidence and rejected otherwise. This approach significantly reduces the computational burden of verification while it still provides rigorous mathematical guarantees on the correctness of the model.

We apply the PAC verification framework for robustness certification. Consider an algorithm that \textit{estimates} the global robustness risk. For any predetermined $\epsilon, \delta \in (0,1]$, a PAC algorithm is capable of producing an output $\widehat{\mathbf{R}}_{rob}(f_{\theta},m)$ such that
\begin{equation}
\mathbb{P}(
    \mathbf{R}_{rob}(f_{\theta},m) - \epsilon
<   \widehat{\mathbf{R}}_{rob}(f_{\theta},m)
<   \mathbf{R}_{rob}(f_{\theta},m) + \epsilon
) \geq 1-\delta.
\label{pac_spec}
\end{equation}

The algorithm $\widehat{\mathbf{R}}_{rob}$ in Equation~\ref{pac_spec} can be fulfilled by a standard Monte-Carlo-based PAC verification algorithm. The algorithm draws random samples from $\mathcal{D}$ and queries the DNN as a black box through forwarding passes. From Hoeffding's inequality~\cite{hoeffding1994probability}, $\mathcal{O}(\frac{1}{\epsilon^{2}})$ samples are sufficient to ensure that the error is smaller than $\epsilon$ with high probability.
We will present a more efficient verification algorithm in Section~\ref{sec:algorithm}.

The ideal hypothesis $\mathbf{R}_{rob}(f_{\theta},m) = 0$ is impractical for typical DNNs. Instead, one should propose a hypothesis that is achievable according to the selected robustness metric $m$. Following the discussion about Equation~\ref{approx}, the optimal values for the robustness risk are
\begin{equation}
  \mathbf{R}^\star_{rob}(f_{\theta},m) =
    \begin{cases}
      \ m_0: \mathbf{R}_{c}(f_{\theta}), & \\
      \ m_1: \mathbf{R}_{c}(f_{\theta}), & \\
      \ m_2: 0. & 
    \end{cases}    
\label{optimal_values_cases}
\end{equation}

The above values are reached when $\mathbf{R}_{b}(f_{\theta})=0$, \textit{i.e.}, $\textit{DB}_{f_{\theta}}$ has no intersection with $\mathcal{D}$, and $\mathbf{R}_{gb}(f_{\theta})=0$, \textit{i.e.}, $\textit{DB}_c$ is excluded from $\mathcal{D}$. These conditions are made possible in verification because the distribution $\mathcal{D}$ specified by an ideal probabilistic program is sufficiently distant from ground truth boundaries.

\subsection{Global Robustness Certification and Measurement}
\label{sec:measure}

\subsubsection{Global Robustness Criterion.} Even with the above-mentioned relaxations, a typical DNN is not globally robust with respect to Expression~\ref{risk}. We allow extra flexibility from two aspects. On the one hand, for each input sample, we tolerate violations to the \textit{local} robustness metric if the occurrence is less than a threshold $t$. It is reasonable because regularization techniques like noise injection~\cite{he2019parametric} and randomized smoothing~\cite{cohen2019certified} 
can mitigate both white-box and black-box attacks when $t$ is sufficiently small.
On the other hand, 
we let the verifier prescribe an additional error $\rho$. The relaxed global robustness criterion thus becomes
\begin{equation}
\mathbf{R}_{rob}(f_{\theta},m,t) \leq \mathbf{R}^\star_{rob}(f_{\theta},m) + \rho,
\label{criterion}
\end{equation}
where
\begin{equation}
\mathbf{R}_{rob}(f_{\theta},m,t) \triangleq \mathbf{E}_{x \sim \mathcal{D}}\mathbbm{1}[\mathbf{E}_{x' \sim \mathbb{B}(x,r)} \mathbbm{1}[m(f_{\theta}, x, x')=\mathtt{false}] > t].
\label{acceptable}
\end{equation}

We term $t$ as \textit{local robustness threshold}, $\mathbf{E}_{x' \sim \mathbb{B}(x,r)} \mathbbm{1}(m(f_{\theta}, x, x')=\mathtt{false})$ as \textit{local robustness risk}, and $\rho$ as \textit{acceptable error}. Notice that when both $t$ and $\rho$ are set to $0$, the global robustness risk reduces to the regular definition as described in Expression~\ref{risk}, and is required to reach the optimal values in Expression~\ref{optimal_values_cases}.

\subsubsection{Cumulative Robustness Function.} Based on the global robustness criterion, we propose the cumulative robustness function as a comprehensive measure of the global robustness of a DNN. The function is given by 
\begin{equation}
R(t) = 1 - \mathbf{R}_{rob}(f_{\theta},m,t)
\label{cumulative}
\end{equation}
and is monotonically increasing with $t$. When evaluated at a certain $t$ value, the function returns the probability that a random sample in $\mathcal{D}$ is robust for local robustness thresholds up to $t$. A curve can be efficiently plotted as $t$ continuously changes with fixed $f_\theta$ and $m$, referred to as the \textit{cumulative robustness curve}.

\subsubsection{Parameter Estimation of Local Robustness Risk.}

As shown in Expression~\ref{acceptable}, an estimation of $\mathbf{R}_{rob}(f_{\theta},m,t)$ needs to be conducted in two levels. In a naive Monte Carlo approach, the estimator draws $N$ i.i.d. samples $g_1, \cdots, g_{N}$ from the global distribution $\mathcal{D}$, and $M$ i.i.d. samples from the $r$-neighbourhood of every $g_i$ to determine whether the local robustness risk is below the threshold. However, violating the robustness metric locally could be an event with a non-zero but extremely small probability. It is likely that no realisation of the event can be encountered with a reasonable choice of $M$.

We aim to devise a simple parameter estimation method to predict the local robustness risk. Define the \textit{margin} of scores
\begin{equation}
  h(x, x', m) =
    \texttt{max}_{j\neq i}\ y_j(x') -  y_{i}(x'),
\end{equation}
where
\begin{equation}
  \begin{cases}
    \ m_0: i = c(x'), & \\
    \ m_1: i = c(x), & \\
    \ m_2: i = f_{\theta}(x). & 
  \end{cases}
\end{equation}

In a nutshell, the margin function evaluates the difference between the score of the reference class and the highest score among other classes. $x'$ violates the robustness metric if the corresponding margin is greater than $0$.

With a little abuse of notation, we use $\mathbb{P}(\mathcal{N}(0,1) > k)$ to denote the portion of the standard normal random variable that is greater than $k$. Within the $r$-neighbourhood of a sample $x$, if the margins $h(x, x', m)$ are normally distributed, the following result holds:

\begin{prop}
\label{prop_normal}
Let $x$ be an input sample and $x'$ be drawn from $\mathbb{B}(x,r)$ uniformly at random. If the probability distribution of $h(x, x', m)$ follows a normal distribution $\mathcal{N}(\mu_{h},\sigma_{h}^2)$, the local robustness risk $\mathbf{E}_{x' \sim \mathbb{B}(x,r)} \mathbbm{1}[m(f_{\theta}, x, x')=\mathtt{false}]$ is equal to $\mathbb{P}(\mathcal{N}(0,1) > -\mu_{h} / \sigma_{h})$.
\end{prop}

\begin{proof}
The robustness metric $m$ is violated at $x'$ when the value of $h(x, x', m)$ is greater than 0. The probability of violation is thus $\mathbb{P}(\mathcal{N}(\mu_{h},\sigma_{h}^2) > 0) = \mathbb{P}(\mathcal{N}(0,1) > -\mu_{h} / \sigma_{h})$.
\end{proof}

We observe that if the scores are taken before the softmax layer, the distribution of $h(x, x', m)$ is close to a normal distribution when $r$ is small. We give an intuitive explanation as follows.
Let $x$ be an image that has $d$ pixels, $x' \sim \mathbb{B}(x,r)$ be the perturbed image, and $\Delta = x' - x$.
Both $x'$ and $\Delta$ can be decomposed with respect to their input dimensions: $x' = [x'_1 ,  \cdots , x'_d]$, $\Delta = [\Delta_1, \cdots, \Delta_d]$. 
Consider the difference of margins for these two images when $|\Delta|$ is close to zero: $h(x, x', m) - h(x, x, m) \approx
\Delta_1 \cdot \frac{\partial h}{\partial x'_1} \big|_{x'=x} + ... +
\Delta_d \cdot \frac{\partial h}{\partial x'_d} \big|_{x'=x}$. 
If the gradients of the loss function are regularized, we can expect that all the partial derivatives in the above formula are bounded when training is completed.
Using the weighted central limit theorem~\cite{weber2006weighted}, for a sufficiently large $d$, the distribution of $h(x, x', m) - h(x, x, m)$ approximates a normal distribution. Note that $h(x, x, m)$ is a constant for a specific $x$. Therefore, under reasonable assumptions, the distribution of $h(x, x', m)$ is approximately normal. We empirically find that this approximation becomes more accurate as $r$ decreases.

\subsubsection{Adaptive Multi-level Splitting.}

Although the above method requires significantly fewer simulations than the naive Monte Carlo, its accuracy is of concern when the local robustness risk is extremely small. Advanced Monte Carlo techniques, such as importance sampling and importance splitting, can construct unbiased estimates with reduced variance for extremely rare events and find counterexamples when they exist. Among those, adaptive multi-level splitting (AMLS)~\cite{cerou2019adaptive} can exploit the margin function to estimate the precise value of local robustness risk.

Specifically, when applied to our problem, AMLS partitions the margins with an ascending sequence of levels $L_1 \cdots L_k$. Every iteration of AMLS starts with $M$ perturbed samples all satisfying $h(\cdot) > L_{i-1}$. During an iteration, the algorithm first decides $L_{i}$ in an adaptive way, such that exactly $M_0$ samples satisfy $h(\cdot) > L_{i}$. It then split those $M_0$ samples through a Markov process to produce a total of $M$ new samples, all satisfying the new condition. The algorithm terminates when $L_{i} \geq 0$. A similar technique was applied to measure the local robustness risk for $m_2$ in~\cite{webb2018statistical}. The authors demonstrated the strengths of AMLS, including its scalability to large neural networks and its reliability in providing high-quality estimates.

However, it is intractable to apply AMLS on all $N$ samples in the global distribution due to the computational cost.
Note that a large $N$ is required for PAC verification (Expression~\ref{pac_spec}).

\subsection{The Global Robustness Certification Algorithm}
\label{sec:algorithm}

Section~\ref{sec:measure} discusses two methods, namely parameter estimation and AMLS, to estimate the local robustness risk. In this section, we present the AMLS Calibrated parameter Estimation (ACE) algorithm. ACE combines the benefits of both methods to obtain efficient yet accurate estimates of local robustness risks. 
Essentially, ACE adopts AMLS to calibrate the results from parameter estimation.
More explicitly, ACE assumes that there exists a strong relationship between the probability of $h(x, x', m) > 0$ and the local robustness risk at $x$. 
Hence, it randomly selects a subset of samples to learn this relation with respect to $f_{\theta}$ and $\mathcal{D}$.
ACE then utilizes the learned model to predict the local robustness risks for the remaining samples, and finally constructs the cumulative robustness function with all the predictions.

\begin{algorithm}[htb]
\caption{Global Robustness Certification Algorithm}\label{alg:cert}
\begin{flushleft}
        \textbf{Inputs:} DNN classifier $f_{\theta}$, probabilistic program generator $G$, maximal radius of adversarial perturbation $r$, number of samples $N$, number of perturbations per sample $M$, number of AMLS executions $N_0$, robustness metric $m$. \\
        \textbf{Output:} Cumulative robustness function $R(t)$.
\end{flushleft}
\begin{flushleft}
\begin{algorithmic}[1]
\State Sample $g_1, \cdots, g_{N}$ i.i.d. from $G$
\label{ppl}
\For{\texttt{i $=$ 1 to $N$}}
\label{stat_begin}
\For{\texttt{j $=$ 1 to $M$}}
\State Sample $g'_{i,j}$ uniformly at random from $\mathbb{B}(g_i,r)$
\State $z_{i,j} \leftarrow h(g_i, g'_{i,j}, m)$
\EndFor
\State Compute $\mu_{z,i}$ and $\sigma_{z,i}$, the mean and the standard deviation of $z_{i,1}, \cdots, z_{i,M}$
\EndFor
\label{stat_end}

\For{\texttt{i $=$ 1 to $N_0$}} \label{amls_begin}
\State Compute $p_i \leftarrow $ \textit{local\_AMLS}$(g_i,f_{\theta}, m)$
\EndFor
\label{amls_end}

\For{\texttt{i $=$ 1 to $N_0$}}\label{reg_begin}
\State Set $\widehat{\mu_{i}} \leftarrow log(p_i)$, $\widehat{\sigma_{i}} \leftarrow$ \textit{variance\_estimator}($g_i$)
\State $x_i \leftarrow log(\mathbb{P}(\mathcal{N}(0,1) > -\mu_{z,i} / \sigma_{z,i}))$\label{mu-sigma}
\State $y_i \leftarrow log(p_i)$
\EndFor

\State Run linear regression on $x$ and $y$, \textit{s.t.} $y_i = \beta_0 + \beta_1 x_i + \epsilon_i$
\label{reg_end}

\For{\texttt{i $=$ $N_0 + 1$ to $N$}}\label{pred_begin}
\State $x_i \leftarrow log(\mathbb{P}(\mathcal{N}(0,1) > -\mu_{z,i} / \sigma_{z,i}))$
\State $\widehat{\mu_{i}} \leftarrow \beta_0 + \beta_1 x_i$, $\widehat{\sigma_{i}} \leftarrow$ \textit{variance\_estimator}($g_i$)
\EndFor
\label{pred_end}
\State $R(t) \leftarrow \frac{1}{N} \cdot \sum_{i=1}^N \mathbb{P}(\mathcal{N}(\widehat{\mu_{i}},\widehat{\sigma_{i}}) \leq log(t))$\label{ret_begin}\\
\Return{$R(t)$}
\label{ret_end}

\end{algorithmic}
\end{flushleft}
\end{algorithm}

The global robustness certification algorithm is detailed in Algorithm~\ref{alg:cert}. It first draws $N$ random samples from the global distribution given by the probabilistic program generator (Line~\ref{ppl}). For each of these samples, it produces $M$ perturbed samples in the $r$-neighbourhood. Then the algorithm infers the margins through forwarding passes and computes the mean and the standard deviation (Line~\ref{stat_begin}-\ref{stat_end}). 
Afterwards, AMLS is launched for the first $N_0$ samples (Line~\ref{amls_begin}-\ref{amls_end}).

The algorithm builds a linear model between the distribution of the margins and the local robustness risks using those samples evaluated by AMLS (Line~\ref{reg_begin}-\ref{reg_end}). We choose the independent variable as the portion of the standard normal distribution that is greater than $-\mu_{z,i} / \sigma_{z,i}$ (Line~\ref{mu-sigma}), which are the statistics of the margins surrounding $g_i$. Note that $\mu_{z,i}$, the average margin, is typically negative. Intuitively, when $\mu_{z,i}$ increases, meaning that the score of the target class has a smaller advantage over those of the remaining classes, the chance of violating the local robustness metric increases. On the other hand, as $\sigma_{z,i}$ increases, the chance of violation increases when $\mu_{z,i} < 0$. We will empirically demonstrate the strength of the linear relation in our experiments.

Thereafter, the algorithm utilizes the linear model to generate predictions for the remaining samples that are not evaluated by AMLS (Line~\ref{pred_begin}-\ref{pred_end}). It outputs a cumulative robustness function, which is an integral of the local robustness risk predictions among all samples (Line~\ref{ret_begin}-\ref{ret_end}).

We use standard techniques to estimate the variance of the predictions~\cite{lee2018variance}. However, the choice of variance estimators has an almost negligible impact on the output for a sufficiently large $N$.
\section{Evaluation}

We implemented our global robustness certification algorithm with PyTorch 1.7 and CUDA 11.3. All experiments are conducted on a Linux desktop with a 4-core 3.2GHz processor, GTX 1070, and 16GB RAM. We use $l_{\infty}$-norm to bound perturbations throughout the experiments.

We choose the generative model from Bayesian Program Learning (BPL)~\cite{lake2015human,feinman2020learning}. This model generates human-like handwritten characters in 50 classes, the same as the Omniglot dataset. The generated images are resized into $28\times 28$ pixels and regularized to mimic the Omniglot dataset.

We choose DCN4, a variant of the decoder choice network~\cite{liu2021decoder}, as $f_{\theta}$. We obtained the network parameters from the original authors. The model has superior performance on the Omniglot dataset and is among the top on the global Omniglot challenge ranking in terms of classification accuracy~\cite{DCN_result}.

\subsection{General Performance}

\begin{table}[b]
    \centering
    \setlength{\tabcolsep}{0.5em} 
    \begin{tabular}{c c c  c  c c  c c  c c }
    \hline
        \multicolumn{3}{c}{Setting} & Average & \multicolumn{6}{c}{ \underline{\quad\quad Cumulative Robustness (\%) \quad\quad}} \\
        \cline{1-3}
        \multirow{2}{*}{Method} & \multirow{2}{*}{$N$} 
        & \multirow{2}{*}{$N_0$} & Runtime & 
        \multicolumn{2}{c}{$t=10^{-5}$} &
        \multicolumn{2}{c}{$t=10^{-10}$} &
        \multicolumn{2}{c}{$t=10^{-15}$} \\
        &&& (s) & Mean & SD & Mean & SD & Mean & SD \\
        \hline
        \hline

Naive MC & 25 & $-$ & 7357.0 & 96.7 & 4.5 & $-$ & $-$ & $-$ & $-$ \\
\hline
AMLS & $-$ & 70 & 7261.6 & 94.9 & 2.7 & 87.6 & 3.5 & 76.4 & 4.7 \\
\hline
 & 100 & 20 & 2172.9 & 95.2 & 2.4 & 86.7 & 4.4 & 74.0 & 6.7 \\
 & 100 & 40 & 4283.5 & 94.3 & 2.5 & 85.8 & 3.0 & 73.2 & 6.7 \\
\multirow{2}{*}{ACE} & 100 & 60 & 6511.7 & 95.3 & 2.1 & 88.7 & 3.2 & 76.1 & 5.0 \\
 & 200 & 60 & 6449.7 & 94.9 & 1.3 & 86.8 & 1.5 & 75.0 & 3.3 \\
 & 400 & 60 & 6569.7 & 94.9 & 0.9 & 85.8 & 1.6 & 73.5 & 2.9 \\
 & 800 & 60 & 6877.7 & 95.0 & 0.4 & 86.4 & 0.9 & 73.5 & 2.8 \\

\hline
    \end{tabular}
    \vspace{0.3em}
    \caption{Comparison of execution time and accuracy with baselines.}
    \label{tab:performance}
\end{table}

\subsubsection{Setup.} We evaluate the performance of the ACE algorithm by comparing it with two baseline settings: the naive Monte Carlo (naive MC) and the AMLS only (AMLS). All computations are deployed on the CPU, except that forward passes are deployed on the GPU. We set a soft time limit for each setting, \textit{i.e.}, we force terminate only after the computation of the current sample finishes once the time limit is reached. We let the naive Monte Carlo and the AMLS compute as many samples as possible before hitting the time limit. We choose $m=m_2$ for this experiment.

We set $M = 10^5$ for the naive Monte Carlo because its accuracy will further deteriorate when $M$ is greater, whereas we set $M=200$ for the ACE. 
We configure the \textit{local\_AMLS} function used by both the AMLS and the ACE settings as follows: sample quantile $= 0.1$, number of particles $= 200$, maximal number of levels $= 20$, number of Metropolis-Hastings updates after each level $= 10$, and the adaptive width proposal as described in~\cite{webb2018statistical}.

The results are summarized in Table~\ref{tab:performance}. We repeat the experiment for each setting $30$ times.

\subsubsection{Accuracy.} We evaluate the estimate accuracy at three different local robustness thresholds. 
Because all the methods are unbiased estimators of the robustness risk, the one with a smaller standard deviation (SD) should be more accurate.
It can be seen from Table~\ref{tab:performance} that the naive Monte Carlo has the worst performance. It cannot find even a single counterexample at $t=10^{-10}$ and $t=10^{-15}$ within the designated time period. Among the other settings, accuracy is generally improved when either $N$ or $N_0$ increases. As $N_0$ grows, the linear regression is more reliable, which in turn improves the quality of the local robustness risk predictions. As $N$ grows, more samples are drawn from $\mathcal{D}$, so an estimator can better capture the global distribution and thus improve its accuracy.

\subsubsection{Execution Time.} The naive Monte Carlo is the slowest of the three methods. Among different settings of AMLS and ACE, the execution time is dominated by the \textit{local\_AMLS} function calls, whose cost is in proportion to $N_0$. As a result, running
the ACE algorithm with a high accuracy only requires a reasonable number of samples and a machine time proportional to
the model’s forward propagation. Furthermore, the ACE algorithm can be easily parallelized.

\subsection{Parametric Estimation and Regression}

\begin{figure}[htb]
    \centering
    \begin{subfigure}[b]{0.47\textwidth}
         \centering
         \includegraphics[scale=0.37]{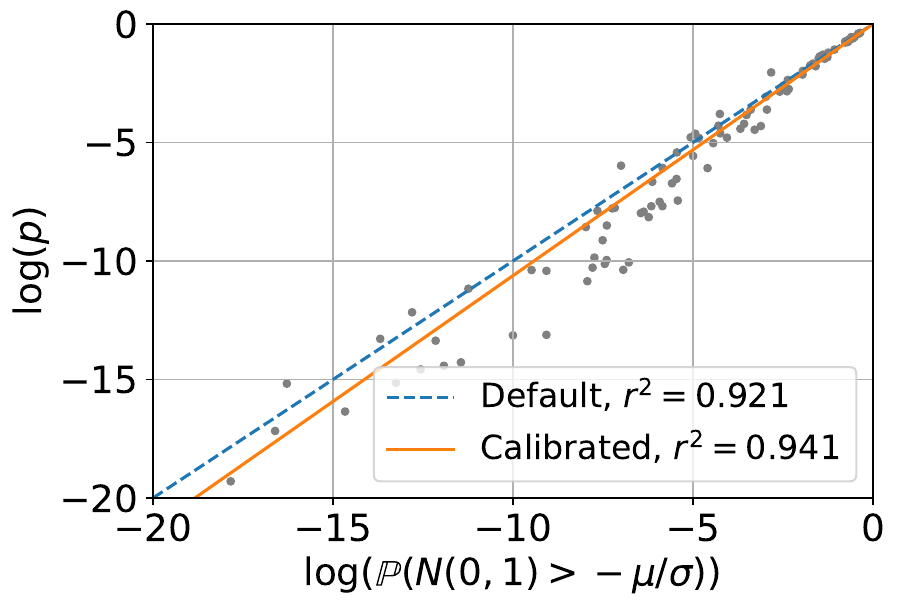}
         \caption{perturbation radius $r=0.05$}
     \end{subfigure}
     \begin{subfigure}[b]{0.47\textwidth}
         \centering
         \includegraphics[scale=0.37]{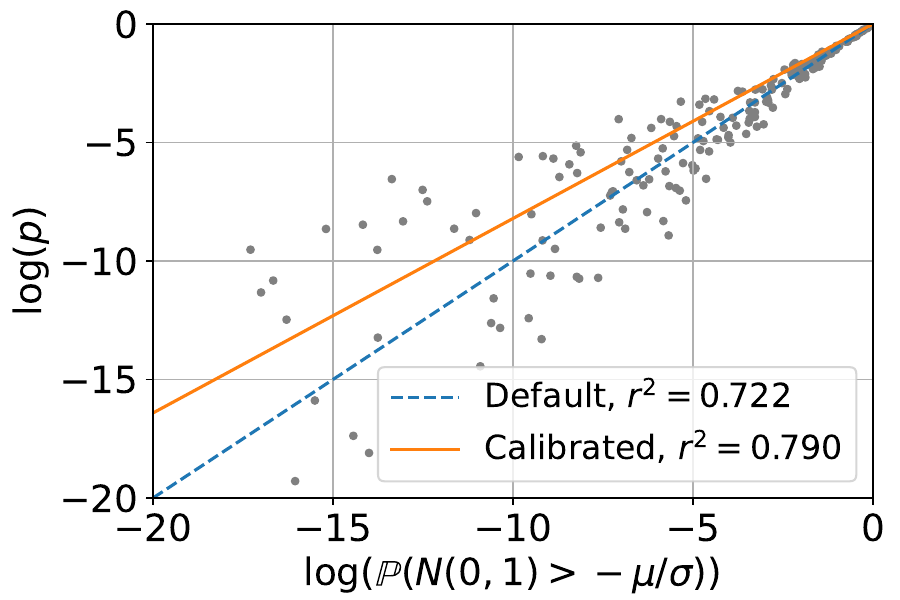}
         \caption{perturbation radius $r=0.1$}
     \end{subfigure}
    \caption{Assessment of parameter estimation with (solid line) and without (dashed line) regression. Local robustness risks (\textit{p}) of the points are computed by AMLS.}
    \label{fig:reg}
\end{figure}

This experiment assesses the performances of parameter estimation and calibration. Each point represents a random sample drawn from $\mathcal{D}$ from all classes. Computing the ground truth local robustness risks for these samples is prohibitively expensive, so we use the results from \textit{local\_AMLS} instead.

Figure~\ref{fig:reg}(a) and Figure~\ref{fig:reg}(b) show our experiment results for $r=0.05$ and $r=0.1$, respectively. Without any calibration, the parameter estimation predicts $p$ from $\mathbb{P}(\mathcal{N}(0,1) > -\mu_{h} / \sigma_{h})$ and achieves high $\mathrm{r}^2$ values. Moreover, after launching \textit{local\_AMLS} on $40$ samples and calibration, the $\mathrm{r}^2$ values on the remaining samples are further improved.

Note that the local robustness risks are displayed on a logarithmic scale, so the samples with small $p$ values seem to deviate from the prediction lines. The method is more reliable for smaller perturbations. This is partly because the outputs of a DNN tend to be continuous in a smaller region, so the distribution of the margins is closer to a normal distribution.

\subsection{Cumulative Robustness Function}

\begin{figure}[htb]
    \centering
    \begin{subfigure}[b]{0.47\textwidth}
         \centering
         \includegraphics[scale=0.37]{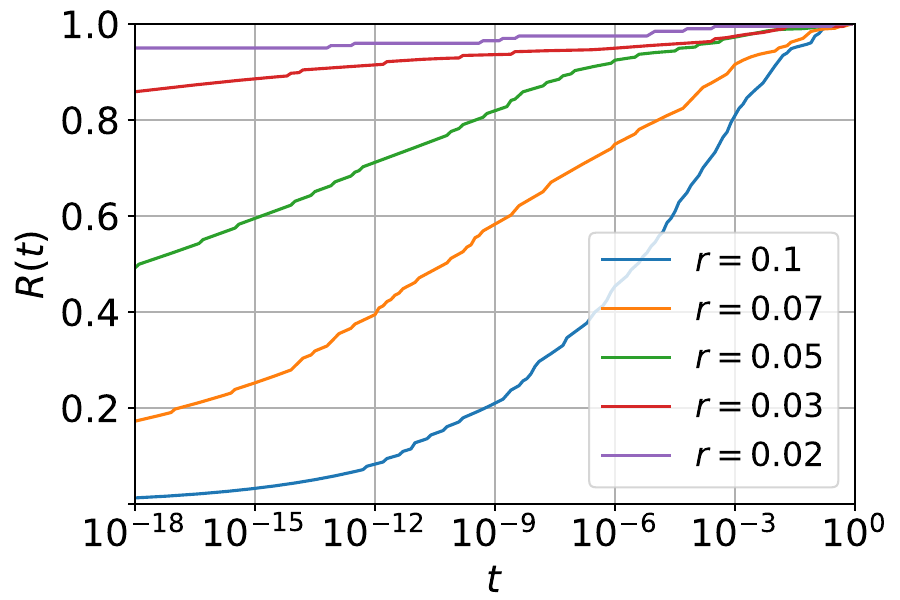}
         \caption{}
     \end{subfigure}
     \begin{subfigure}[b]{0.47\textwidth}
         \centering
         \includegraphics[scale=0.37]{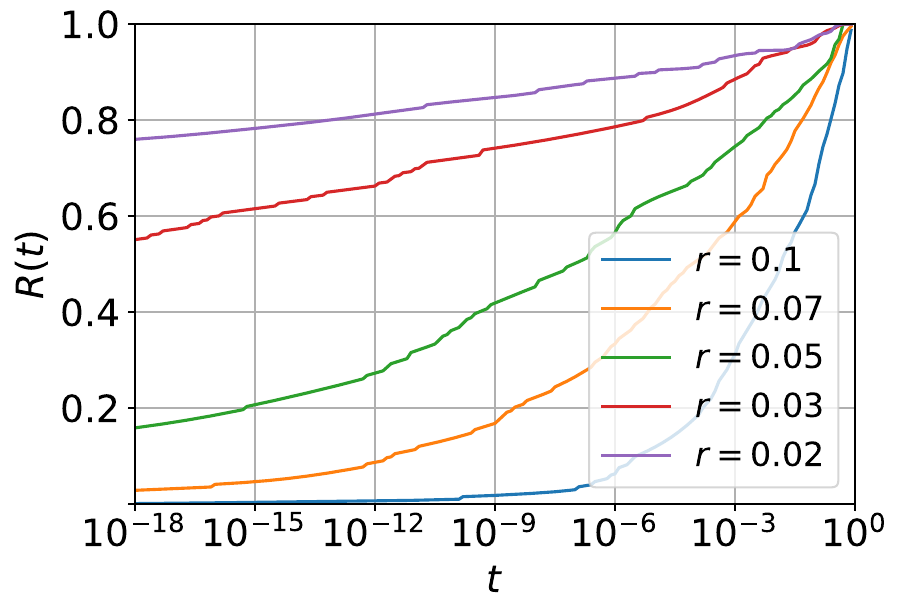}
         \caption{}
     \end{subfigure}
    \caption{The cumulative robustness functions for classes 1-5 (a) and 6-10 (b).}
    \label{fig:curves}
\end{figure}

Figure~\ref{fig:curves} (a) and (b) plot the cumulative robustness functions of different classes in the Omniglot dataset and different perturbation radii. Each curve represents an $R(t)$ function yielded from a single  execution of Algorithm~\ref{alg:cert}. All curves are monotonically increasing as the value of the local robustness threshold $t$ increases. A verifier can prescribe a $t$ value to check whether the intersection on the curve is above her expected value for global robustness.

\subsection{Mining Counterexamples}

\begin{figure}[htb]
    \centering
    \includegraphics[scale=0.45]{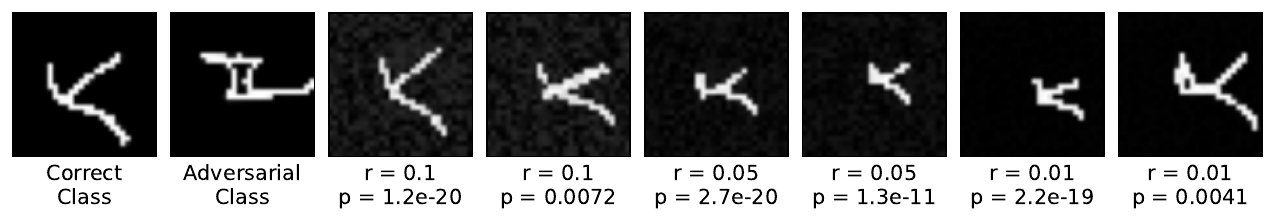}
    \includegraphics[scale=0.45]{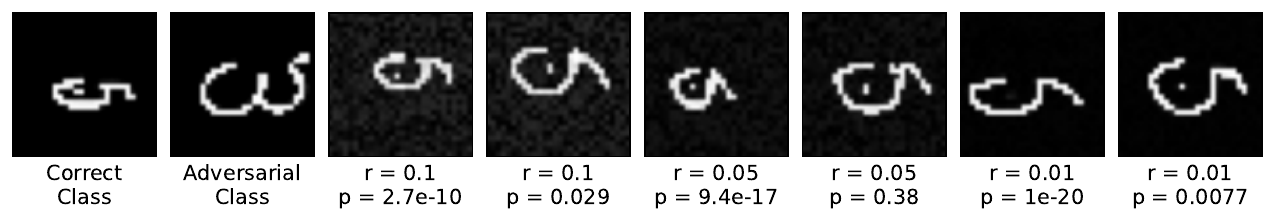}
    \includegraphics[scale=0.45]{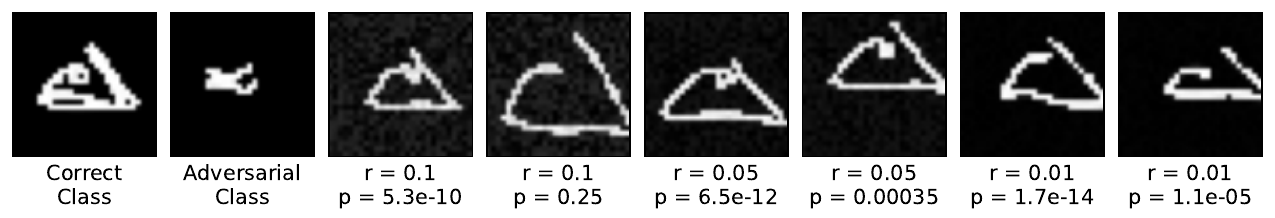}
    \caption{Counterexamples mis-classified to designated classes.}
    \label{fig:cex}
\end{figure}

Figure~\ref{fig:cex} showcases ACE's capability to find rich, diversified and human understandable counterexamples. Each pair of counterexamples with the same perturbation radius $r$ are picked from $10$ counterexamples, and they have the least and the greatest $p$ values in their neighborhoods. ACE can produce as many such counterexamples as one wishes for adversarial training.
\section{Related Work}

\subsection{Certification of Deep Neural Networks}

The DNN verification problem can be formulated in the following general form~\cite{goodman2014concepts}: given $\mathcal{X} \subset \mathbb{R}^{m}$ and $\mathcal{Y} \subset \mathbb{R}^{n}$, deciding whether $x\in \mathcal{X}\Rightarrow f_{\theta}(x)\in \mathcal{Y}$. To certify local robustness, $\mathcal{X}$ is the $l_p$-norm bounded neighborhood of an input example $x_0$, while $\mathcal{Y}$ corresponds to the label $f_{\theta}(x_0)$.
Various approaches are proposed to formally solve the verification problem. Those methods include layer-by-layer reachability analysis~\cite{gehr2018ai2,xiang2018output,weng2018towards,huang2017safety,boopathy2019cnn}, SAT or SMT reachability analysis~\cite{bunel2017unified,ehlers2017formal,katz2017reluplex}, mixed-integer linear programming~\cite{bastani2016measuring,tjeng2017evaluating}, dual optimization~\cite{wong2018provable,raghunathan2018certified}, and semi-definite optimization~\cite{fazlyab2020safety,raghunathan2018semidefinite}.
These approaches have several drawbacks. As already mentioned, the verification scope is limited to a single $l_p$-bounded neighborhood of an input example, while the bound is selected without basis, and this neighbourhood can overlap with the decision region of another class. Furthermore, these approaches cannot be extended to a global region directly. 
They employ abstraction-based techniques to deal with the size of the network and the non-linearity in the activation functions.
The over-approximation caused by abstraction will be significantly amplified if the input set is a global region instead of a small neighbourhood. In consequence, the verification accuracy will deteriorate if the input set is the decision region of a whole class.

\subsection{Probabilistic Methods for Deep Neural Network Verification}

While deterministic certification approaches aim to analyze the worst case of a DNN when the input is within a small region, probabilistic approaches relax the worst case requirement and can thus scale to large networks.
The \textit{margin} is defined as $y_i(x)-y_j(x)$, where $i$ is the ground truth label and $j$ is the label of the targeted class. The margin is typically required to be greater than a given positive value so that the DNN cannot be attacked by an adversary.
Probabilistic upper and lower bounds for the event can be computed layer-by-layer analytically when the input is constrained within an $l_p$-ball and follows a known distribution~\cite{weng2019proven}. Alternatively, the event can be approximated by a set of linear functions. The coefficients of those functions can be learned by drawing random input examples and inferring the corresponding values of the margin function. Scenario optimization guarantees that the learned linear model approximates the original DNN in terms of the event with high probability. In addition to certifying the correctness of the DNN, the learned model can be utilized to compute maximal input perturbation~\cite{anderson2020certifying} or to mine counterexamples~\cite{li2021probabilistic}.  Although this approach is agnostic to input distribution, the input should be bounded within an interval.

Probabilistic robustness can also be formulated as the probability that the \textit{Lipschitz continuity property} holds. The property can be estimated by randomly drawing pairs of samples~\cite{mangal2019robustness} or modelling the whole DNN as a probabilistic program consisting of  conditional affine transformations and then executing program verification~\cite{mangal2020probabilistic}.
Noticeably, ~\cite{leino2021globally} proposes \textit{global Lipschitz bounds}, which differs from our definition of global robustness. For example, while the Lipschitz property ensures a bounded change of class scores for some given perturbation, it does not ensure an unchanged classification result, especially when the input is close enough to the model's decision boundary.

Some probabilistic methods can verify general correctness properties, including robustness properties, for DNNs. A Binary Neural Network can be encoded into conjunctive normal form (CNF). The probability that a correctness property in CNF holds on the network can be estimated by a model counter~\cite{baluta2019quantitative}. 
Advanced statistical algorithms including Multi-level splitting~\cite{webb2018statistical} and adaptive hypothesis testing~\cite{baluta2019quantitative} can find counterexamples and estimate the probability of violation even if violations of the target property are extremely rare.
However, none of the methods above address the problem of generalizing the verification to the whole decision regions of classes.

\section{Conclusion}

We propose a comprehensive global robustness certification method called ACE. Based on a close investigation of robustness metrics in global distributions, we use probabilistic program generators and a sequence of estimation and regression techniques to enhance the method. ACE is capable of efficiently and accurately estimating global robustness as a function of perturbation radius and local robustness threshold. A verifier can query the function to check if a DNN meets the user-specified robustness requirements. Additionally, it can produce a large amount of high-quality data for adversarial training.

\clearpage

\bibliographystyle{splncs}
\bibliography{youl}
\end{document}